\documentclass[11pt]{article}
\usepackage[paper=letterpaper,margin=1in]{geometry}

\usepackage{amsmath,amssymb,amsfonts}
\usepackage{fullpage}

\usepackage{comment}
\usepackage{xspace}
\usepackage{bbm}
\usepackage{algorithm}
\usepackage{algpseudocode}

\newenvironment{proof}{\noindent{\bf Proof : \ }}{\hfill$\Box$\par\medskip}

\newtheorem{theorem}{Theorem}
\newtheorem{corollary}[theorem]{Corollary}
\newtheorem{lemma}[theorem]{Lemma}

\newtheorem{definition}[theorem]{Definition}

\newenvironment{proofof}[1]{\begin{trivlist} \item {\bf Proof
			#1:~~}}
	{\qed\end{trivlist}}
\renewenvironment{proofof}[1]{\par\medskip\noindent{\bf Proof of #1: \ }}{\hfill$\Box$\par\medskip}

\newcommand{\COMMENTED}[1]{{}}

\newcommand{\sav}{\textsc{SAV}\xspace}
\newcommand{\ig}{\textsc{IG}\xspace}
\newcommand{\rig}{\textsc{RIG}\xspace}

\newcommand{\eps}{\ensuremath{\varepsilon}}

\newcommand{\ex}[1]{\mathbb{E}\left[#1\right]}

\newcommand{\E}{\mathbb{E}}
\newcommand{\cavg}{c_{\text{avg}}}
\newcommand{\favg}{f_{\text{avg}}}

\newcommand{\p}{\ensuremath{\mathbb{P}}}
\newcommand{\R}{\ensuremath{\mathbb{R}}}

\newcommand{\pigreedy}{\pi_\text{greedy}}

\newcommand{\dom}{{\sf dom}}

\DeclareMathOperator*{\argmax}{arg\,max}
\DeclareMathOperator*{\argmin}{arg\,min}

\newcommand{\1}{{\mathbbm{1}}}

\newcommand{\sub}{\preccurlyeq}
\renewcommand{\p}{p}

\title{Adaptivity in Adaptive Submodularity}

\author{Hossein Esfandiari \\Google Research\and Amin Karbasi\\ Yale University \and Vahab Mirrokni\\Google Research}

\date{}

\begin{document}\sloppy
	\maketitle
\begin{abstract}
Adaptive sequential decision making is one of the central challenges in machine learning and artificial intelligence. In such problems,  the goal is to design an interactive policy that plans for an action to take, from a finite set of $n$ actions, given some partial observations. It has been shown that in many applications such as active learning, robotics,  sequential experimental design, and active detection, the utility function  satisfies adaptive submodularity, a notion that generalizes  the notion of diminishing returns to policies. 
In this paper, we revisit the power of adaptivity in maximizing an adaptive 
monotone submodular function. We propose an efficient semi adaptive policy that with 
$O(\log n \times\log k)$ adaptive rounds\footnote{We also refer to an adaptive 
round as a batch query and use these two terms interchangeably throughout the 
paper.} of observations can achieve an almost tight $1-1/e-\eps$ approximation 
guarantee with respect to an optimal policy that carries out $k$ actions in a fully 
sequential manner. To complement  our results, we also show that it is 
impossible to achieve a constant factor approximation with $o(\log n)$ adaptive 
rounds. We also extend our result to the case of adaptive stochastic minimum 
cost coverage where the goal is to reach a desired utility $Q$ with the cheapest 
policy. 
We first prove the conjecture of the celebrated work of Golovin and Krause \cite{golovin11} by showing that the greedy policy achieves the asymptotically tight logarithmic approximation guarantee without resorting to stronger notions of adaptivity. We then propose a semi adaptive policy that provides the same guarantee in polylogarithmic adaptive rounds through a similar information-parallelism scheme. Our results shrink the adaptivity gap in adaptive submodular maximization by an exponential factor.

\end{abstract}
\sloppy
\section{Introduction}

Adaptive stochastic optimization under partial observability is one of the fundamental challenges 
in 
artificial 
intelligence and machine learning with a wide range of applications, including active 
learning 
\cite{dasgupta2008hierarchical}, 
optimal 
experimental design \cite{sebastiani1997bayesian},
interactive recommendations \cite{ICML2012Karbasi_455}, viral marketing 
\cite{singer2016influence}, adaptive influence maximization 
\cite{tong2016adaptive}, active detection \cite{chen2014active}, Wikipedia link 
prediction 
\cite{mitrovic2019adaptive}, 
and 
perception 
in 
robotics \cite{javdani2014near}, to name a few. In such problems, one needs to 
adaptively 
make a sequence of 
decisions while taking into account the stochastic observations collected in 
previous rounds. For instance, in active learning,  the goal is 
to 
learn a 
classifier by carefully requesting as few labels as possible 
from a set of unlabeled data points.  Similarly, in experimental design, a practitioner may 
conduct a series of tests in order to reach a conclusion. 

Even though it is possible to determine all the selections 
ahead of time before any observations take place (e.g., 
select all the data points at once or conduct all the medical 
tests simultaneously), so called 
\textit{a priori selection},  it is more efficient to consider a 
\textit{fully adaptive}  
procedure that exploits the information obtained from past 
selections in order to make a new selection. Indeed, a priori 
and fully sequential selections are simply two ends of a 
spectrum. 
In this paper, we develop a semi-adaptive policy that enjoys the power of a fully sequential procedure while performing exponentially fewer adaptive rounds compared to previous work. In particular, we only need poly-logarithmic number of rounds for both adaptive stochastic submodular 
maximization and adaptive 
stochastic minimum cost coverage problems. In the following, we will state these problems more formally, and then present our results in more details.

\subsection{Notations}
We mostly follow the notation used by Golovin and Krause \cite{golovin11}. Let the \textit{ground set} $E = \{e_1, \dots, e_n\}$ be a finite set of 
elements (e.g., tests in medical diagnostics, data points in active learning). Each 
element $e\in E$ is associated with a random variable $\Phi(e)\in \Omega$ where $\Omega$ is the 
set of 
all possible outcomes. A realization of the random  variable $\Phi(e)$ is 
denoted by 
$\phi(e)\in\Omega$. Note that a realization $\phi: E\rightarrow\Omega$ is simply a function  from the ground set $E$ to the outcomes $\Omega$. For the ease of notation, we can also  represent  $\phi$ as a relation $\{(e,\omega): \phi(e)=\omega, \forall e\in E \}$. For instance, in medical diagnosis, the element $e$ may represent a test, such as
the blood pressure, and 
$\Phi(e)$ its outcome, such as, high or low. Or in active learning, an item $e$ may represent  an 
unlabeled data point 
and 
$\Phi(e)$ may represent its label.  We assume that there is a prior probability 
distribution  $\p(\phi) = \p(\Phi=\phi)$ over realizations $\phi$. This probability distribution encodes 
our 
uncertainty about the outcomes as well as their dependencies. In its simplest form, the outcomes maybe 
independent and the distribution $\p$ completely factorizes. The product distribution may very well be a valid model in the sensor placement scenario 
where sensors may fail to work independent of one another \cite{asadpour2016maximizing}. However, in many practical 
settings, such as medical diagnosis and active learning,  the underlying distribution may not factorize and the outcomes may depend on each other. 

In this paper, we consider adaptive strategies for picking elements where based on our observations so far, we sequentially pick an item $e$ and observe its associated outcome $\Phi(e)$. The set of observations made so far can be represented by a \textit{partial realization} $\psi = \{(e,\omega): \psi(e)=\omega\}\}\subseteq E\times \Omega$. We use $\dom(\psi) = \{e: \exists \omega \text{ s.t } (e,\omega)\in \psi\}$ to denote the domain of $\psi$.  We say that $\psi$ is a \textit{subrealization} of $\psi'$, and denoted by $\psi\sub\psi'$, if $\dom(\psi)\subseteq \dom(\psi')$ and $\forall e\in \dom(\psi)$ we have $\psi(e) = \psi'(e)$. Similarely, a partial realization $\psi$ is \textit{consistent} with a realization $\phi$, and denote by $\psi \sub \phi$, if they agree everywhere in the domain of $\psi$.  We take a Bayesian approach and assume that after observing $\phi$, we can compute the posterior distribution $\p(\phi|\psi) = \p(\Phi=\phi|\psi\sub\Phi)$.

%By overloading the notation, we  use $\Phi$ to 
%denote the set of random variables indexed by the elements of $E$ and $\phi$ their realizations. In particular,  we refer to $\phi$ as a realization.

%Similarly, for a subset of elements $A\subseteq E$ we use $\Psi_A$ to 
%denote the set of random variables indexed by the elements of $A$ and $\psi_A$ their realizations. We refer to $\psi_A$ as a partial realization. We define $\dom(\psi_A)=A$ and drop $A$ from the index when it is clear from the context.

%By overloading the notation, we  use $X_A$, for any subset $A\subseteq E$,  to denote the set of random variables indexed by the elements of $A$ and $x_A$ their realizations. 
%We assume that there is a joint probability 
%distribution  $\p(\Phi)$ over the set of random variables $\Phi$. This probability distribution encodes 
%our 
%uncertainty about the outcomes as well as their dependencies. In its simplest form, the outcomes maybe 
%independent and the distribution $\p$ completely factorizes, i.e., a product distribution over the set of 
%random variables $\Phi$. The product distribution may very well be a valid model in the sensor placement scenarios 
%where sensors may fail to work independent of one another. However, in many practical 
%settings, such as medical diagnosis and active learning,  the underlying distribution may not factorize and the outcomes may depend on each other. 

A \textit{policy}
$\pi: 2^{E\times \Omega} \rightarrow E$ is a partial 
mapping from partial observations $\psi$ to elements $E$, stating which element $e\in E$  to select next\footnote{Golovin and Krause \cite{golovin11} originally defined a policy as follows $\pi: 2^{E}\times {\Omega}^{E} \rightarrow E$. However, in subsequent works \cite{chen2013near, chen2015sequential}, the less restrictive form, the one we consider  in this paper, is used. }. Note that any deterministic policy can be visualized by a decision tree. In the proofs we also make use of two notions related to policies, namely, truncation and concatenation \cite{golovin11}. Given a policy $\pi$, we define the \textit{level-$k$-truncation} $\pi_{[k]}$  by running $\pi$ until it terminates or until
it selects $k$ items. Given two policies $\pi_1$ and $\pi_2$, we define \textit{concatenation} $\pi_1 @ \pi_2$ as the policy obtained by first running $\pi_1$ to completion, and then running policy $\pi_2$ as if from a fresh start, ignoring the information gathered during the running of $\pi_1$.     The 
utility of a set of 
observations $\psi$ is specified  through a utility function $f:2^{E\times \Omega} \rightarrow \R_{+}$.
%
% where 
%clearly  $f(\psi_A)$ depends on the realization of the random variable $\psi_A$ and the chosen set of 
%elements 
%$A$.
 The expected utility of a policy is then defined as 
$$f_{avg}(\pi)  = \E[f(S(\pi,\Phi), \Phi)] = \sum_{\phi} \p(\phi)f(S(\pi,\phi),\phi),$$
where the expectation is taken  with respect to $p(\phi)$.  Throughout the paper  $S(\pi,\phi)$ denotes the set of elements selected by policy $\pi$ under  realization $\phi$. Without 
any structural assumptions, it is known that finding an optimal  policy, the one that maximizes the expected 
utility,
 is notoriously 
hard as  in many cases the utility functions are  
computationally 
intractable
\cite{papadimitriou1987complexity}. 

Adaptive submodularity \cite{golovin11}, a  generalization of diminishing returns property from sets 
to policies, is a 
sufficient condition under which a partially observable stochastic optimization problem admits 
(approximate) tractability. This condition ensures that the \text{expected} marginal benefit 
associated with any particular  selection never increases as we make more
observations. More formally, we define the \textit{conditional expected marginal benefit} $\Delta(e|\psi)$ of an item $e$ conditioned 
on 
observing the partial realization $\psi$ as follows:
$$\Delta(e|\psi) \doteq\E[f(\dom(\psi)\cup \{e\}, \Phi) - f(\dom(\psi), \Phi)|\psi\sub \Phi] = \sum_{\omega}\p(\Phi(e) = \omega|\psi) [f(\psi\cup \{e,\omega\})-f(\psi)]$$ 
%
%Similarly, the conditional expected marginal utility of a policy $\pi$ is defined as
%$$\Delta(\pi|\psi) = \E[f(\dom(\psi)\cup \{e\}, \Phi) - f(\dom(\psi), \Phi)|\Phi\sim\psi]$$
%
%
%We say that $\psi_A$ is a subrealization of $\psi_B$, denoted by $\psi_B \succcurlyeq \psi_A$, if $A\subseteq B$ and for all $e\in A$, $\psi_A(e) = \psi_B(e)$. 
The utility function $f$ is \textit{adaptive submodular} if for all subrealizations 
$\psi 
\sub \psi'$, and all $e\in E\setminus \dom(\psi')$, we have $$\Delta(e|\psi)\geq \Delta(e|\psi').$$ Moreover, we say that 
the utility function $f$ is \textit{adaptive monotone} if for all subrealizations $\psi$, and all 
$e\notin \dom(\psi)$ we 
have $\Delta(e|\psi)\geq 0$.

Whenever we use expectation notation $\E[\chi]$ for a random variable $\chi$, 
the expectation is over all randomness of $\chi$, unless specified otherwise. 
Moreover, note that we always use capital letters for random variables, and small 
letters for realizations. For example $\Psi$ refers to a random variable, and 
$\psi$ refers to a realization of $\Psi$, and hence $\psi$ is a deterministic 
quantity.
%
% More formally, let $E = \{e_1, \dots, e_n\}$ be a finite set of actions/items. Each 
%action $e\in E$ is associated with a random variable $X_e\in \Omega$ where $\Omega$ is the 
%set of 
%all possible outcomes. A realization of the random  variable $X_e$ is denoted by 
%$x_e\in\Omega$. For instance, in medical diagnosis, an action $e$ may represent a test and 
%$X_{e}$ its outcome. Or in active learning, an item $e$ may represent  an unlabeled data point 
%and 
%$X_e$ its label.  More generally, we  use $X_A$, for any subset $A\subseteq E$,  to 
%denote the set of random variables indexed by the elements of $A$ and $x_A$ their realizations. 
% We assume that there is a joint probability 
%distribution  $\p$ over the set of random variables $X_E$. 

\subsection{Problem Formulation}

The general goal in adaptive stochastic optimization is to develop policies that can maximize the 
expected utility while minimizing the cost of running the policy. One way to formalize it is 
through the \textit{adaptive stochastic submodular maximization} problem where we aim to maximize the expected utility 
subject to a cardinality constraint, i.e., 
$$\pi^* = \argmax_{\pi} f_{avg}(\pi) \qquad \text{s.t.} \qquad |S(\pi, \phi)|\leq k  \qquad\text{ whenever } 
\p(\phi)>0.$$
It is known that when the utility is adaptive submodular and adaptive monotone, the  
greedy policy, shown in Algorithm~\eqref{algo:greedy},
achieves the tight $(1-1/e)$ approximation ratio with respect to the optimal policy 
\cite{golovin11}. This 
result has lead to a surge of applications in decision making problems that are amenable to 
myopic optimization 
such as active learning \cite{golovin2010near}, 
interactive 
recommender systems \cite{karbasi2012comparison}, value of information 
\cite{chen2015submodular}, and active object detection \cite{chen2014active}, to name 
a few. 
\begin{algorithm}[H]
	\centering
	\caption{Adaptive Greedy Policy $\pigreedy$ \cite{golovin11}\label{algo:greedy} for Adaptive Stochastic Maximization} 
	\begin{algorithmic}[1]
		\Require Ground set $E$, size $k$, distribution $p(\phi)$, function $f(\cdot)$
		% \Procedure{\bmm}{}
		\State {\bf initialize}  $A\leftarrow \emptyset$, $\psi\leftarrow \emptyset$
		\For{$i=1$ to $k$}
		\State $e^* = \arg\max_{e\in E\setminus A} \Delta(e|\psi)$
		\State $A\leftarrow A\cup\{e^*\}$
		\State $\psi\leftarrow \psi \cup\{(e^*, \Phi(e^*))\}$
		\EndFor
		\State \Return A
		
		% \EndProcedure
	\end{algorithmic}    
\end{algorithm}

An 
alternative 
formalization is through \textit{adaptive stochastic minimum cost coverage} where we prespecify a 
quota $Q$ of utility to achieve, and aim to find a policy that achieves it with the cheapest policy, i.e.,  
$$ \pi^* = \argmin_{\pi} \cavg (\pi) \qquad \text{s.t.} \qquad f_{avg}(\pi)\geq Q \qquad\text{ whenever 
} 
\p(\phi)>0, $$
where $\cavg (\pi)= \E_{\p}[|S(\pi, \phi)|] $ is the expected number of actions a policy $\pi$ selects. We can also consider  a slightly more general setting where each item $e$ has a non-negative cost $c(e)$ and replace $\cavg (\pi)= \E_{\p}[c(S(\pi, \phi))]$ where $c(S) = \sum_{e\in S} c(e)$. 
Unlike the adaptive stochastic submodular maximization problem, the performance of the greedy policy, shown in Algorithm~\eqref{algo:greedycoverage}, is unknown for 
the above problem unless one makes strong assumptions about the distribution or the utility 
function. One of the contributions of this paper is to resolve this issue.
\begin{algorithm}[H]
	\centering
	\caption{Adaptive Greedy Policy $\pigreedy$ \cite{golovin11}\label{algo:greedycoverage} for Adaptive Stochastic Minimum Cost Coverage} 
	\begin{algorithmic}[1]
		\Require Ground set $E$, quota $Q$, cost function $c(\cdot)$, distribution $p(\phi)$, function $f(\cdot)$
		% \Procedure{\bmm}{}
		\State {\bf initialize}  $A\leftarrow \emptyset$, $\psi\leftarrow \emptyset$
		\While{$f(A,\psi)<Q$}
		\State $e^* = \arg\max_{e\in E\setminus A} \Delta(e|\psi)/c(e)$
		\State $A\leftarrow A\cup\{e^*\}$
		\State $\psi\leftarrow \psi \cup\{(e^*, \Phi(e^*))\}$
		\EndWhile
		\State \Return A
		
		% \EndProcedure
	\end{algorithmic}    
\end{algorithm}

\subsection{Our Contributions}

Fully sequential policies benefit from previous observations in order to make informed decisions. 
In many scenarios,  however, it is more effective (and sometimes the only way) to select multiple 
elements in parallel and  observe their realizations together. Examples include crowdsourcing (where 
a single task consists of a collection of unlabeled data to be labeled altogether), multi-stage viral marketing (where in each stage 
a subset of nodes are chosen as seed nodes), batch-mode pool-based active learning (where the label of a set of data 
points are requested simultaneously), or medical diagnosis (where there is a shared cost among experiments). 
A batch-mode,  semi-adaptive policy is a mix of  a priori and fully sequential selections.  
The focus of this paper is to answer the following question in the context of adaptive stochastic 
optimization:
\begin{quote}
How many adaptive rounds of observations are needed in order to be competitive  to  an optimal 
and 
fully sequential policy?
\end{quote}
%
%
% Even though there exists effective heuristics for batch-mode adaptive 
%submodular maximization \cite{ }, such as greedily and in an off-line fashion selecting batches 
%of 
%fixed sizes,  there is little known 
%about their theoretical performance with respect to the fully sequential policy beyond the 
%product 
%distribution. 
We answer the above question in the context of adaptive submodularity.
%
% in both  adaptive 
%stochastic coverage setting and adaptive stochastic minimum cost coverage 
%setting. For the adaptive 
%stochastic coverage, we develop \algpar, a policy that is competitive to an optimal and fully sequential 
%policy  but rather 
%surprisingly   performs 
%exponentially 
%fewer 
%sequential rounds. Crucially, \algpar  does not commit to  fixed size batches of items in each 
%round. In contrast, it 
%determines
%the 
%number of elements to select  based on the information gathered 
%from 
%the previous 
%observations. Indeed, we show that any policy that commits  to batches of fixed size $L$ may suffer ... We 
%then extend the information-parallelism framework to the adaptive stochastic minimum cost coverage 
%setting and develop \algparcov, a policy that is competitive to an optimum sequential policy while 
%performing exponentially fewer sequential rounds. We also complement our result by showing a lower 
%bound on the number of sequential adaptive rounds. 
In this paper, we consider two  adaptive stochastic optimization problems, namely, adaptive stochastic maximization  
 and 
adaptive stochastic minimum cost cover. We re-exam the required amount 
of adaptivity in order to be competitive to the optimal and fully sequential policy. In particular, we 
show the following results in the information-parallel stochastic optimization when the utility function 
is adaptive submodular and adaptive monotone.
\begin{itemize}
	\item For the adaptive stochastic submodular 
	maximization problem, we develop a semi adaptive policy %\algpar 
	that with $O(\log(n)\log(k))$ adaptive rounds (a.k.a., batch queries) achieves 
	the  tight 
	$(1-1/e-\eps)$ approximation guarantee with
	respect to the optimum policy $\pi^*$ that selects $k$ items fully sequentially, i.e., $\favg(\pi)\geq 
	(1-1/e-\eps)\favg(\pi^*)$. 
	\item We complement the above result by showing that no policy can achieve a constant factor 
	approximation guarantee with fewer than $o(\log (n))$ adaptive rounds. Moreover, the 
	approximation guarantee of any semi adaptive policy that chooses batches of fixed size 
	$r$ will degrade with a factor of $O(r/\log^2(r))$.
	\item For the adaptive stochastic minimum cost coverage problem, we show that the 
	greedy policy achieves an asymptoticly tight logarithmic approximation guarantee, effectively proving 
	\cite{golovin11}'s conjecture. More precisely, we show that $\cavg(\pigreedy)\leq 
	(\cavg(\pi^*) +1)\log\left(\frac{n 
	Q}{\eta}\right)+1$ where we make the common
	assumption that there is a value $\eta$ such 
	that $f(\psi)>Q-\eta$ implies that $f(\psi)=Q$ for all partial realizations $\psi$.
	\item We also develop a semi adaptive policy %\algparcov 
	for the the adaptive 
	stochastic minimum cost coverage problem that achieves 
	the same logarithmic approximation guarantee 
	with $O\big(\log n \log (Qn/\eta)\big)$ adaptive rounds.
\end{itemize}

\section{Related Work}
Submodularity captures an intuitive diminishing returns property where the gain of adding an 
element 
to a set decreases as the set gets larger. More formally, a non-negative set function $f: 2^{V} 
\rightarrow \R_{+}$ is \textbf{submodular} if for all sets $A \subseteq B \subset V$ and 
every element $e \in  V \setminus B$, we have $$f(A \cup \{e\}) - f(A) \geq f(B \cup \{e\}) - f(B).$$ 
Submodular maximization has  found numerous applications in machine learning and 
artificial intelligence \cite{tohidi2020submodularity}, including neural network interpretation \cite{elenbergDFK17}, data 
summarization \cite{lin2011class}, crowd teaching \cite{singla2014near}, privacy 
\cite{mirzasoleiman2017deletion}, 
fairness \cite{celis2016fair}, 
and adversarial attacks in deep neural nets 
\cite{lei2018discrete}. Moreover, in many information gathering and sensing  scenarios, 
the 
objective functions 
satisfy submodularity \cite{krause08efficient, wei2015submodularity, guillory2012active}. 
However, the classic 
notion of submodularity 
falls short in  interactive 
information acquisition settings as it requires the decision maker to commit 
to all of her selections 
ahead of time, in an open-loop fashion \cite{guillory10interactive}. 

To circumvent this issue, \cite{golovin11} proposed adaptive 
submodularity, a generalization of submodularity from sets 
to 
policies. Like submodularity, adaptive submodularity is a sufficient 
condition that ensures tractability in adaptive settings. More 
precisely, in the adaptive stochastic submodular maximization problem, when the 
objective function is adaptive monotone and 
adaptive submodular,  the greedy policy achieves the tight 
$(1-1/e)$ approximation guarantee with respect to an optimum 
policy \cite{golovin11}. More generally, \cite{gotovos2015non} proposed a random greedy 
policy that not only retains the aforementioned $(1-1/e)$ 
approximation ratio 
in the monotone setting, but also  
provides a $(1/e)$ approximation ratio for the non-monotone adaptive 
submodular functions.  

The results for adaptive stochastic 
minimum cost coverage problem are much weaker. Originally,  
\cite{golovin11} claimed that the greedy policy also achieves a  
logarithmic approximation factor but as pointed out by \cite{nan2017comments} the proof 
was flawed.  Instead, under stronger 
conditions, namely, strong adaptive submodularity and strong adaptive 
monotonicity, Golovin and Krause proposed a new proof, with a  
squared-logarithmic factor approximation, i.e.,  $\cavg(\pigreedy)\leq 
\cavg(\pi^*) \left(\log\left(\frac{n 
	Q}{\eta}\right)+1\right)^2$. Note that there are some fundamental technical differences between the notions of submodularity and adaptive submodularity. For example, submodularity preserves under truncation, that is, if a function $f(\cdot)$ is submodular, for any constant $c$, $\min\big(f(\cdot),c\big)$ is submodular as well. This comes very handy in designing algorithms for submodular functions and often used as a simple way to reduce minimum cost coverage to submodular maximization. However, unfortunately, truncation does not preserve adaptive submodularity (See Appendix~\ref{appx:trunc}) and thus all the previous attempts to use this reduction are futile. In this paper, we prove the original 
	conjecture of Golovin and Krause \cite{golovin11} and show that under adaptive submodularity (without 
	resorting to stronger conditions), the greedy policy achieves a  
logarithmic approximation factor, namely, $\cavg(\pigreedy)\leq 
(\cavg(\pi^*) +1)\log\left(\frac{n 
	Q}{\eta}\right)+1$.

The main focus of this paper is to explore the  information parallelism, a.k.a.,
batch-mode, stochastic optimization \cite{gupta2017adaptivity, goemans2006stochastic}. Many active learning 
problems naturally fall into this setting when it is more 
cost-effective to request labels in large batches, rather than 
one-at-a-time (for detailed discussions, we refer the interested reader to \cite{chen2017near}). Note that the two extremes of batch-mode 
stochastic optimization are 
full batch setting (i.e., all selections are done in a single batch, and hence the batch-mode setting 
reduces to the non-adaptive, open-loop optimization problem) and 
full sequential setting (i.e., elements are selected one-by-one in a closed-loop manner where each 
selection is based on the results of all previous 
selections). In this paper, we lay out a rigorous foundation for the semi-adaptive setting 
where elements 
are 
selected in a sequential and closed-loop way but with multiple selections at each round. 

There are a 
few partial results regarding the semi-adaptive policy for the adaptive stochastic minimum 
cost coverage problem. In particular, \cite{chen2013near} proposed a  policy that selects 
batches of fixed size $r$ and proved that under strong adaptive submodularity and strong 
adaptive monotonicity, this  policy achieves a poly-logarithmic 
approximation  to an optimal policy that is also constrained to picking up
batches of size $r$. Note that this result does not provide any guarantees with respect to 
the actual baseline, namely, the optimal and fully sequential policy. 
Moreover,  \cite{chen2017near} showed that this policy has a sublinear-approximation\footnote{Unfortunately, the approximation factor grows polynomially in $r$. Moreover, note that this result assumes strong adaptive submodularity and strong adaptive monotonicity.} guarantee against the fully sequential policy.
In fact, we show that for 
the adaptive stochastic submodular maximization problem, the approximation factor of a fixed batch-policy 
suffers by at least a factor of $\log^2(r)/r$ in the worst case, so unless $r$ is a fixed 
constant, no constant factor approximation guarantee is possible. 

Back to the adaptive 
stochastic 
minimum cost coverage, when the distribution $\p$ is fully factorized (i.e., the outcomes 
are independent),
\cite{agarwal2019stochastic} very recently  showed 
that there exists a policy that, using $O(\log(Q)/\log\log(Q))$ rounds of adaptivity, 
achieves a poly-logarithmic approximation to the optimal sequential policy. 
In this paper, we propose %\algparcov, 
a (batch-mode) semi adaptive policy that, using only polylogarithmic adaptive rounds, achieves an asymptotically tight logarithmic 
approximation to the fully sequential policy for general  adaptive monotone submodular functions 
(we do not need to resort to stronger notions of adaptivity and monotonicity). 

To the best of our knowledge, no results are known for semi-adaptive policies for the  adaptive 
stochastic submodular maximization problem. We  develop 
%\algpar, 
a  (batch-mode) semi adaptive policy that achieves an almost tight 
$1-1/e-\eps$ approximation guarantee with only polylogarithmic adaptive rounds. We also complement our 
result by showing that no semi-adaptive policy can achieve a constant factor approximation to the 
optimal policy by fewer than $o(\log(n))$ adaptive rounds. 

Our work is also related to the adaptivity complexity of  submodular maximization, which refers 
to the number of parallel rounds required to achieve a constant factor approximation guarantee in 
the offline, open-loop setting.  \cite{balkanski2018adaptive} developed a parallel algorithm that 
$O(\log n)$ rounds  finds a solution with an 
approximation arbitrarily close to $\frac{1}{3}$ which was soon improved to $(1 - 
\frac{1}{e} - \epsilon)$-approximation \cite{fahrbach2018submodular, 
balkanski2018exponential,ene2018submodular}. The adaptivity complexity was also studied in the 
non-monotone submodular maximization \cite{balkanski2018nonmonotone, fahrbach2018nonmonotone, 
chen2018unconstrained}, convex minimization \cite{balkanski2018parallelization, diakonikolas2020lower,bubeck2019complexity} and multi-armed bandit \cite{gao2019batched, esfandiari2019batched}. We lift the notion of adaptivity complexity from the offline optimization  to  
the interactive 
setting where instead of parallelizing the optimization steps we parallelize the information 
acquisition.

\section{Greedy Versus Optimum}

Throughout the paper,  we assume that the utility function  $f$ is an adaptive 
monotone and adaptive submodular function with respect to the distribution 
$p(\phi)$. 
In this section, we show how the expected utility obtained by the greedy policy $\pigreedy$ is related to the expected utility obtained by the optimum policy $\pi^*$. Let us define $\pi^{\tau}$ to be the policy that runs the greedy policy $\pigreedy$ and stops when the expected marginal gain of all of the remaining elements is less than or equal to $\tau$. We define $\tau_i$ to be a threshold such that the expected number of elements selected by $\pi^{\tau_i}$ is $i$, in other words $\sum_{\phi}p(\phi)|S(\pi^{\tau_i},\phi)|=i$.\footnote{Use an arbitrary tie breaking rule to make it exactly equal. For example, if by accepting elements whose expected marginal benefit is strictly larger than $\tau$ the greedy policy selects $\alpha<i$ elements in expectation, and by accepting elements with expected marginal  benefit larger than or equal to $\tau$ the greedy policy selects $\beta>i$ elements in expectation, then we let the policy $\pi^{\tau_i}$ accepts $\tau$ with probability $\frac{i - \alpha}{\beta - \alpha}$. Hence, the expected number of items accepted by $\pi^{\tau_i}$ is $\alpha + \frac{i - \alpha}{\beta - \alpha} (\beta-\alpha) = i$.} 
Also remember that for two policies $\pi$ and $\pi'$ we define $\pi@\pi'$  to be 
a policy that first runs $\pi$ and then runs $\pi'$ from a fresh start (i.e., ignoring 
the information gathered by $\pi$). This definition implies 
$S(\pi@\pi',\phi)=S(\pi,\phi)\cup S(\pi',\phi)$.
%\textcolor{red}{(what is the difference between $\phi$ and $\psi$). I think we used $\phi$ for realizations and $\psi$ for partial realizations. But then in the text they are not consistently used. There is a difference between partial realizations and realizations}.
The following is the key lemma of this paper.
\begin{lemma}\label{lm:1-1/e^l}
	For any policy $\pi^*$ and any positive integer $\ell$ we have 
	\begin{align*}
	f_{avg}(\pi^{\tau_{\ell}}) > \big(1-e^{-\frac{\ell}{\ex{K}+1}}\big) f_{avg}(\pi^*),
	\end{align*}
	where $K$ is a random variable that indicates the 
	number of items picked by $\pi^*$, i.e. $K=|S(\pi^*, \Phi)|$. 
\end{lemma}

%\begin{comment}
Before proving Lemma~\ref{lm:1-1/e^l}, we provide some primitives that we use in the proof of this lemma as well as Lemmas~\ref{lm:par:1-1/e^l} and~\ref{lm:batch:1-1/e^l}. For a randomized policy $\pi$, we use $\Theta_{\pi}$ to indicate the random bits of the policy $\pi$. We use $\theta_{\pi}$ to indicate a realization of $\Theta_{\pi}$ and use $p(\theta_{\pi})$ to indicate the probability that $\theta_{\pi}$ is realized. We drop $\pi$ from the notation when it is clear from the context.

For a deterministic policy $\pi$, a (potentially) randomized policy $\pi^*$, an element $e$, and two subrealization $\psi$ and $\psi \sub \psi' $, we define the event $<\psi,\psi',e,\pi,\pi^*>$ to be the event that
\begin{itemize}
	\item $\dom(\psi)\subseteq \dom(\pi)$, meaning that the policy $\pi$ selects all the elements of $\psi$, 
	\item and, the set of elements selected by the policy $\pi@\pi^*$, at some point during its run, coincides exactly with the domain of $\psi'$, 
	\item and, right after the policy $\pi@\pi^*$ selects all the elements of $\psi'$, it chooses $e$.
\end{itemize}

Note that for a fixed $\theta_{\pi^*}$, and conditioned on $\psi\sub\Phi '$, we can simulate $\pi@\pi^*$ and deterministically indicate whether the event $<\psi,\psi',e,\pi,\pi^*>$ happens or not.\footnote{Note that if policy $\pi@\pi^*$ attempts to query an element that does not exist in $\psi'$, we know that the event $<\psi,\psi',e,\pi,\pi^*>$ does not happen and do not need to simulate the policy any further.} We define $\1_{\psi',\psi,e,\pi,\pi^*,\theta_{\pi^*}}$ to be a \emph{deterministic} binary variable that is $1$ if and only if, conditioned on $\Theta_{\pi^*}=\theta_{\pi^*}$ and $\psi\sub\Phi '$, the event $<\psi,\psi',e,\pi,\pi^*>$ happens. We use the shorthand $\1_{\psi',\psi,e,\theta}$ (and drop the notations of policies) when it is clear from the context.

%In the following lemma use the notations $\1$ and $\theta$ for policy $\pi^{\tau_i}@\pi^*$, for a certain $i$, hence for the simplicity of notation we drop the notation of the policy from the subscripts and use $\theta$ and $\1_{\psi,e,\theta}$. In other words, we define $\theta:=\theta_{\pi^{\tau_i}@\pi^*}$ and $\1_{\psi,e,\theta}:=\1_{\psi,e,\pi^{\tau_i}@\pi^*,\theta_{\pi^{\tau_i}@\pi^*}}$.

\begin{comment}
We define $\1_{\pi, e,\psi}$ to be a binary random variable that is $1$ if the event $<\pi, e,\psi>$ happens and is $0$ if the event $<\pi, e,\psi>$ does not happen.  \textbf{Note that}, conditioned on $\Phi\sim \psi$, the value of $\1_{\pi,e,\psi}$ is independent of the randomness of $\Phi$ and solely depends on the randomness of policy $\p$. For example, if $\pi$ is a deterministic policy, conditioned on $\Phi\sim \psi$, the value of $\1_{\pi,e,\psi}$ is deterministic. 

We define the random variable $\gamma(e|\psi)$ to be $f(\psi'\cup \Phi(e))-f(\psi')$ conditioned on $\Phi\sim \psi'$ and $\1_{e,\psi}=1$. Recall that, $\Delta(e|\psi)= \ex{f(\psi\cup \Phi(e))-f(\psi)\Big|\Phi\sim \psi}$. Hence, we have
\begin{align}
\ex{\gamma(e|\psi)}&=
\ex{f(\psi'\cup \Phi(e))-f(\psi')|\Phi\sim \psi' \text{ and } \1_{e,\psi}=1}\nonumber\\
&=\ex{f(\psi'\cup \Phi(e))-f(\psi')|\Phi\sim \psi'}\nonumber\\
&=\Delta(e|\psi).
\end{align}
\end{comment}

\begin{proofof}{Lemma~\ref{lm:1-1/e^l}}
	First we provide an upper bound on $f_{avg}(\pi^*)$. Pick an arbitrary number $i\in \{0,\dots,\ell\}$. Note that by adaptive monotonicity, we have $f_{avg}(\pi^*) \leq f_{avg}(\pi^{\tau_i}@\pi^*)$. Next we show that $f_{avg}(\pi^{\tau_i}@\pi^*) \leq f_{avg}(\pi^{\tau_i})+\ex{K}\Delta^{\tau_i}(\pi)$, where $\Delta^{\tau_i}(\pi)=f_{avg}(\pi^{\tau_{i}})-f_{avg}(\pi^{\tau_{i-1}})$. This implies that 
	\begin{align}\label{eq:main}
	f_{avg}(\pi^*) \leq f_{avg}(\pi^{\tau_i}@\pi^*) \leq f_{avg}(\pi^{\tau_i})+\ex{K}\Delta^{\tau_i}(\pi).
	\end{align}
	
	Note that once we run $\pi^{\tau_i}$ the set of selected elements and their realizations depends on $\Phi$. To capture this randomness, we let $\Psi_i$ be a random variable that indicates the partial realization observed by running policy $\pi^{\tau_i}$. We use $\psi_i$ to indicate a realization of the random variable $\Psi_i$. Note that, 
	by definition of $\pi^{\tau_i}$ we have ${\tau_i}\geq \max_{e\in E\setminus 
		\dom(\psi_i)} \Delta(e|\psi_i)$. Moreover, by adaptive submodularity, for  
all sub-realizations $\psi_i \sub \psi' $, and for all $e\in 
	E\setminus\dom(\psi')$ we have $\Delta(e|\psi_i)\geq \Delta(e|\psi')$. 
	Therefore, we have
	
	\begin{equation}\label{eq:greedE}
	{\tau_i} \geq \Delta(e|\psi'), \quad\quad{\forall{\psi_i \sub \psi' }\text{ and }\forall{e\in 
			E\setminus\dom(\psi')} }.
	\end{equation} 
	Now, we can bound the different between the expected utility obtained by $\pi^{\tau_i}@\pi^*$ and $\pi^{\tau_i}$ as follows:
	\begin{align*}
		&f_{avg}(\pi^{\tau_i}@\pi^*) - f_{avg}(\pi^{\tau_i}) \\
		&= \sum_{\psi_i\in \Psi_i}\sum_{\psi_i \sub \psi' }\sum_{e\notin \dom(\psi')}\sum_{\theta\in \Theta_{\pi^*}} p(\psi')p(\theta)\1_{\psi_i,\psi',e,\theta}\ex{f(\psi'\cup \Phi(e))-f(\psi')|\psi'\sub\Phi  \text{ and } \Theta_{\pi^*}=\theta} &\text{Definition}\\ 
		&= \sum_{\psi_i\in \Psi_i}\sum_{\psi_i \sub \psi' }\sum_{e\notin \dom(\psi')}\sum_{\theta\in \Theta_{\pi^*}} p(\psi')p(\theta)\1_{\psi_i,\psi',e,\theta}\ex{f(\psi'\cup \Phi(e))-f(\psi')|\psi'\sub\Phi} &\hspace{-1cm}\text{Independency\footnotemark}\\ 
		&=  \sum_{\psi_i\in \Psi_i}\sum_{\psi_i \sub \psi' }\sum_{e\notin \dom(\psi')} \sum_{\theta\in \Theta_{\pi^*}} p(\psi')p(\theta)\1_{\psi_i,\psi',e,\theta}\Delta(e|\psi') &\text{Definition}\\
		&\leq \sum_{\psi_i\in \Psi_i}\sum_{\psi_i \sub \psi' }\sum_{e\notin \dom(\psi')} \sum_{\theta\in \Theta_{\pi^*}} p(\psi')p(\theta)\1_{\psi_i,\psi',e,\theta}\tau_i &\text{Inequality \ref{eq:greedE}}\\
		&\leq \sum_{\psi_i\in \Psi_i}\sum_{\psi_i \sub \psi' }\sum_{e\notin \dom(\psi')} \sum_{\theta\in \Theta_{\pi^*}} p(\psi')p(\theta)\1_{\psi_i,\psi',e,\theta}\Delta^{\tau_i}(\pi) & \Delta^{\tau_i}(\pi) \geq \tau_i  \\
		&= \Big(\sum_{\psi_i\in \Psi_i}\sum_{\psi_i \sub \psi'}\sum_{e\notin \dom(\psi')} \sum_{\theta\in \Theta_{\pi^*}} p(\psi')p(\theta)\1_{\psi_i,\psi',e,\theta}\Big)\Delta^{\tau_i}(\pi) \\
		& = \ex{K}\Delta^{\tau_i}(\pi).
	\end{align*} \footnotetext{Note that $\psi'$ is fixed, hence $f(\psi'\cup\Phi(e))-f(\psi')$ is independent of $\Theta$.}
	This proves inequality \eqref{eq:main} as promised.	 Let us define $\Delta_i^*=f_{avg}(\pi^*)-f_{avg}(\pi^{\tau_i})$. Inequality 
	\ref{eq:main} implies that $\Delta_{i}^*\leq 
	\ex{K}(\Delta_{i-1}^*-\Delta_{i}^*)$. By a simple rearrangement we have 
	$$\Delta_{i}^*\leq\big(1-\frac 1 {\ex{K}+1}\big)\Delta_{i-1}^*.$$ By iteratively 
	applying this inequality we have $$\Delta_{\ell}^*\leq\big(1-\frac 1 
	{\ex{K}+1}\big)^{\ell}\Delta_{0}^*\leq e^{-\frac{\ell}{\ex{K}+1}} \Delta_{0}^*.$$ 
	By applying the definition of $\Delta_{i}^*$ and further rearrangements we have 
	$$f_{avg}(\pi^{\tau_{\ell}}) > \big(1-e^{-\frac{\ell}{\ex{K}+1}}\big) 
	f_{avg}(\pi^*),$$ as desired.
\end{proofof} 
Next, we will use Lemma~\ref{lm:1-1/e^l} to prove the conjecture by 
\cite{golovin11} that the greedy policy achieves the asymptotically tight 
logarithmic approximation guarantee.

\section{Adaptive Stochastic Minimum Cost Coverage}\label{sec:setcover}
In this section, we prove that the greedy policy, outlined in Algorithm~\ref{algo:greedycoverage}, achieves a logarithmic approximation guarantee for  adaptive stochastic minimum cost coverage. For the ease of presentation, we focus on the   unit cost case, i.e., $c(e)=1$ for all $e\in E$. The generalization to the non-uniform cost is immediate.

%To prove this theorem we use Lemma~\ref{lm:1-1/e^l} and follow the usual proof for set cover. We present the proof of this theorem in Appendix~\ref{sec:omitted}

\begin{theorem}\label{thm:setcovermain}
	%Assume that there exists a value $Q$ such that $f(E,\phi)=Q$ for all $\phi$ and 
	Assume that there is a value $\eta \in (0,Q]$ such that $f(\psi)>Q-\eta$ 
	implies $f(\psi)=Q$ for all $\psi$. Let $\pi^*$ be an arbitrary 
	policy (including the optimum policy\footnote{One can think 
	of 
	this as an optimal policy that minimizes the expected number of selected 
	items and guarantees that every realization is covered.}) that covers everything 
	i.e., $f(\pi^*)=Q$ for all $\phi$. %\textcolor{red}{(we have to be consistent, $\phi$ or $\psi$?)}. 
	Let $\pigreedy$ be the greedy policy. We have
	\begin{align*}
		{c_{avg}(\pigreedy)} \leq \big(c_{avg} (\pi^*)+1\big)\log\Big(\frac{nQ}{\eta}\Big)+1,
	\end{align*}
	where $n=|E|$.% is the size of the ground set.
\end{theorem}

\begin{proof}%{Theorem~\ref{thm:setcovermain}}
	Let $K$ be a random variable that indicates the number of items picked by 
	$\pi^*$, i.e., $K=S(\pi^*, \Phi)$. Set $\ell=(\ex{K}+1)  \log (nQ/\eta)$.
	Note that by definition of $\pi^*$ we have $f(\pi^*)=Q$ for all $\phi$, hence we have $f_{avg}(\pi^*) =Q$.
	By Lemma \ref{lm:1-1/e^l} we have
	\begin{align*}
	f_{avg}(\pi^{\tau_{\ell}}) &> \Big(1-e^{-\frac{\ell}{\ex{K}+1}}\Big) f_{avg}(\pi^*) &\text{By Lemma \ref{lm:1-1/e^l}} \\
	&= \Big(1-e^{-\log (nQ/\eta)}\Big) f_{avg}(\pi^*) &\text{Since 
	}\ell=\big(\ex{K}+1\big)  \log (nQ/\eta) \\
	&= \left(1-\frac \eta {n Q}\right)f_{avg}(\pi^*)\\
	&= Q - \frac \eta n. &\text{Since }f_{avg}(\pi^*) =Q
	\end{align*} 
	Recall that by definition we have 
	$\ex{f(\pi^{\tau_{\ell}})}=f_{avg}(\pi^{\tau_{\ell}})=Q - \frac {\eta} n$. 
	Moreover, by adaptive monotonicity we have $  f(\pi^{\tau_{\ell}})\leq  f(\phi) = 
	Q$. Hence by Markov inequality with probability $1-1/n$ we have
	$f(\pi^{\tau_{\ell}})>Q-\eta$.
	%\footnote{Note that if we let $\delta=\min_{\phi} 
	%p (\phi)$, for a deterministic policy having %$f(\pi^{\tau_{\ell}}) > Q-\eta$ with 
	%probability greater than $1-\delta$ implies that we always have 
	%$f(\pi^{\tau_{\ell}}) > Q-\eta$.  This 
	%replaces the factor $n$ in the approximation factor with $\frac{1}{\delta}$. 
	%However, since $\delta$ might be exponentially smaller than $\frac 1 n$, we 
	%find the current statement of the theorem more interesting.}
	By definition of $\eta$ this implies that with probability $1-\frac 1 n$  we 
	have $f(\pi^{\tau_{\ell}})=Q$.
	Therefore, with probability $1-1/n$, the policy $\pi^{\tau_{\ell}}$ reaches 
	the utility $Q$ after selecting  $\ell =\big(\ex{K}+1\big)\log (nQ/\eta)$ items 
	in 
	expectation. Otherwise, $\pigreedy$ picks at most all the $n$ items. Hence the 
	expected number of items that $\pigreedy$ picks is upper bounded by
	\begin{align*}
	\left(1-\frac 1 n\right) \times  \big(\ex{K}+1\big)\log\Big(\frac{nQ}{\eta}\Big)  
	+  \frac 1 n \times n  \leq  \big( c_{avg}(\pi^*)+1\big)  
	\log\Big(\frac{nQ}{\eta}\Big)+1.
	\end{align*}
\end{proof}

With slight modifications to the proof of Theorem \ref{thm:setcovermain} we achieve the following corollary. We provide the proof of this result in the appendix.

\begin{corollary}\label{cr:setcovermain2}
	%Assume that there exists a value $Q$ such that $f(E,\phi)=Q$ for all $\phi$ and 
	Assume that there is a value $\eta \in (0,Q]$ such that $f(\psi)>Q-\eta$ 
	implies $f(\psi)=Q$ for all $\psi$. Let $\pi^*$ be an arbitrary 
	policy that covers everything 
	i.e. $f(\pi^*)=Q$ for all $\phi$. %\textcolor{red}{(we have to be consistent, $\phi$ or $\psi$?)}. 
	Let $\pigreedy$ be the greedy policy. We have
	\begin{align*}
	{c_{avg}(\pigreedy)} \leq \big(c_{avg} (\pi^*)+1\big)\log\Big(\frac{Q}{\delta \eta}\Big)+1,
	\end{align*}
	where $\delta=\min_{\phi} 
	p (\phi)$.% is the size of the ground set.
\end{corollary}

\section{Semi Adaptive Stochastic Submodular Maximization}\label{sec:batch:kcover}
In this section, we provide a policy for adaptive stochastic submodular 
maximization that makes only $O(\log n \log k)$ batch queries (a.k.a. adaptive 
rounds). We show that our policy provides a $(1-\frac 1 e - \eps)$ approximate 
solution compare to that of the best fully sequential policy. 
Our policy is based on two notions, \emph{semi-adaptive values} and 
\emph{information gap}. We first provide some intuition and notations, and then 
explicitly define these two notions. At any stage of the algorithm, semi-adaptive 
value of an item $e$ is our estimate of the expected value of selecting item $e$. 
We provide these estimations based on the information of the last batch query 
that we carried out  and the set of  items 
that we are deciding to select but not queried yet. Information gap is our estimate 
of 
the accuracy of the maximum semi-adaptive value. We use the information gap to 
balance between the loss on the performance and the number of batch queries 
that we make. 
\begin{algorithm}[H]
	\centering
	\caption{Semi Adaptive Greedy Policy for Adaptive Submodular Maximization\label{alg:semi}} 
	\begin{algorithmic}[1]
		\Require Ground set $E$, size $k$,  distribution $p(\phi)$, function $f(\cdot)$, value $\epsilon>0$
		% \Procedure{\bmm}{}
		\State {\bf initialize}  $A\leftarrow \emptyset$, $\psi'\leftarrow \emptyset$, $i\leftarrow 1$
		\While{$|A|<k$}
		\While{$\ig(i,\psi')\geq (1-\epsilon)$ and $|A|<k$} 
		\State $e^* = \arg\max_{e\in E\setminus A} \sav(e,i, \psi')$
		\State $A\leftarrow A\cup\{e^*\}$
		\EndWhile
		\State $\psi''\leftarrow$ query all elements in $A\setminus\dom(\psi')$
		\State $\psi'\leftarrow \psi' \cup \psi''$
		\EndWhile
		\State \Return A
		
		% \EndProcedure
	\end{algorithmic}    
\end{algorithm}
We iteratively and greedily \emph{select} elements based on their semi-adaptive values. 
We continue this selection non-adaptively until the information gap decreases to 
$(1-\eps)$. When the information gap drops below $(1-\eps)$ we \emph{query} all the 
selected elements. We call this algorithm \emph{semi-adaptive greedy} and is outlined in Algorithm~\ref{alg:semi}. In this 
section, we use $\pi$ to refer to this policy. We refer to the step $i$ of a policy as 
the time it selects the $i$-th element.  We  use 
$\Psi_i^{\pi}$ to refer 
to the 
(random) partial realization up to and including step $i$ of policy $\pi$. Again, note that the partial 
realization we observe by running $\pi$ depends on $\Phi$. We also use $\psi'$ 
to 
refer to the observed
partial realization of items that have been queried so far.  Since policy $\pi$ is  deterministic,
given $\psi'$, we can deterministically indicate the domain of $\Psi_i^{\pi}$. Therefore, $\dom(\Psi_i^{\pi})$ is deterministic and well specified (while the state of items $e\in \dom(\Psi_i^{\pi})\setminus \dom(\psi')$ is random). We are ready to define the semi-adaptive values and the information gap.
\begin{definition}[Semi-Adaptive Value] At any step $i$ of the policy $\pi$, and given the partial realization $\psi'\sub \Psi_i^{\pi}$, the semi-adaptive 
value of an item $e\in E\setminus\dom(\Psi_i^{\pi})$ is defined as follows:
\begin{align*}
\sav(e,i, \psi')\doteq\E_{\psi' \sub \Psi_i^{\pi}} \big[\Delta(e|\Psi_i^{\pi})\big].
\end{align*}
\end{definition}

Note that the semi-adaptive value of an item $e$ is equal to the expected 
marginal gain of $e$ over  all the unknown random realizations (i.e., not 
in 
$\psi'$). 

\begin{definition}[Information Gap] At any step $i$ of the policy $\pi$, and given the partial realization $\psi'\sub \Psi_i^{\pi}$, the information gap is defined as follows.
	\begin{align*}
	\ig(i,\psi')\doteq\frac
	{\max_{e\notin \dom(\Psi_i^{\pi})} \E_{\psi' \sub \Psi_i^{\pi} }\big[ 
	\Delta(e|\Psi_i^{\pi})\big]}
	{\E_{\psi' \sub \Psi_i^{\pi}}\big[ \max_{e\notin \dom(\Psi_i^{\pi})} 
	\Delta(e|\Psi_i^{\pi})\big]}.
	\end{align*}
\end{definition}
Equipped with these definitions, we show next the utility obtained by the semi adaptive greedy policy, shown in Algorithm~\ref{alg:semi}, along with the total number of batch queries.
\subsection{Performance}
Golovin and Krause \cite{golovin11} showed that a fully 
sequential greedy policy $\pigreedy$ achieves $	f_{avg}(\pi_{\text{greedy}[\ell]}) > 
\big(1-e^{-\frac{\ell}{k}}\big) f_{avg}(\pi^*).$ The next lemma bounds the 
performance of our semi adaptive greedy policy in a similar fashion. We use the 
notions of 
semi-adaptive values and the information gap to prove this lemma.  In the following 
lemma, $\pi$ is the semi-adaptive greedy policy and $\pi_{[\ell]}$ is a policy that 
runs $\pi$ and stops if it selects $\ell$ items.

\begin{lemma}\label{lm:par:1-1/e^l}
	Let $\pi$ be the semi-adaptive greedy policy. For any policy $\pi^*$ and positive integer $\ell$ we have 
	\begin{align*}
	f_{avg}(\pi_{[\ell]}) > \big(1-e^{-\frac{\ell}{k}}-\eps\big) f_{avg}(\pi^*).
	\end{align*} 
\end{lemma}

This is one of the main technical contributions of the paper, but due to the space 
constraint we provide the proof of this lemma in Appendix~\ref{sec:omitted}. The proof of 
this lemma relies on Lemma \ref{lm:1-1/e^l} and uses a similar machinery.

\subsection{Query Complexity}
In this subsection, we bound the number of batch queries of the semi-adaptive 
greedy policy. We define random variable $\Psi'_t$ to be the partial realization 
obtained by the $t$-th batch query (do not confuse it with $\Psi_i^{\pi}$). We use $\psi'_t$ to indicate the realization of 
random variable $\Psi'_t$. The next lemma shows that after any $\log_{\frac 1 
{1-\eps/ 2}} 
\big(\frac{n}{\delta}\big) = O_{\eps}\big(\log\big(\frac{n}{\delta}\big)\big)$ batch 
queries, the maximum expected  marginal benefit drops by a factor $(1-\frac{\eps}{2})$, 
with high probability. We later apply this lemma iteratively for $O(\log k)$ times 
to show that after $O_{\eps}(\log n \log k)$ batch queries, the maximum 
expected marginal gain is vanishingly small.
\begin{lemma}\label{lm:superRound}
	Pick an arbitrary $t$ and fix partial realization $\psi'_t$. Let $\Delta'_t = 
	\max_{e\notin \dom(\psi'_t)} \Delta(e|\psi'_t)$, and let $t^+= t+ \log_{\frac 1 
	{1-\eps/ 2}} \big(\frac{n}{\delta}\big)$. We have $$\max_{e\notin 
	\dom(\Psi'_{t^+})} \Delta(e|\Psi'_{t^+}) \leq \left(1-\frac {\eps} 2\right) \Delta'_t, $$ with 
	probability at least $1-\delta$.
\end{lemma}
\begin{proof}
	For any $t'\geq t$ we use the random variable $S_{t'}$ to indicate the set of 
	elements such that $\Delta(e|\Psi'_{t'}) \geq (1-\frac{\eps}{2})\Delta'_t$. To 
	prove the above lemma, we show that $\ex{|S_{t'}|} \leq (1-\frac{\eps}{2}) 
	\ex{|S_{t'+1}|}$. This together with $|S_t| \leq n$ implies that $\ex{|S_{t^+}|} 
	\leq \delta$ for $t^+ = t+ \log_{\frac 1 {1-\eps/ 2}} \big(\frac{n}{\delta}\big)$. 
	Note that $|S_{t^+}|$ is a non-negative integer, and hence we have $S_{t^+} = 
	\emptyset$ with probability at least $1-\delta$.
	
	Next, we show that $\ex{|S_{t'}|} \leq (1-\frac{\eps}{2}) \ex{|S_{t'+1}|}$. First 
	note that by adaptive monotonicity $e\in S_{t'+1}$ implies $e\in S_{t'}$, and 
	hence we have $S_{t'+1} \subseteq S_{t'}$. In the following, we use the notion 
	of 
	information gap and show that for any element $e\in S_{t'}$, we have $e \notin 
	S_{t'+1}$ with probability at least $\frac{\eps}{2}$. This directly implies 
	$\ex{|S_{t'}|} \leq (1-\frac{\eps}{2}) \ex{|S_{t'+1}|}$ as desired.
	Note that when we query $\Psi'_{t'+1}$ the information gap is at most 
	$(1-\eps)$. Hence, for some $\Psi_i^{\pi}$ (which corresponds to $\Psi'_{t'+1}$) 
	we have
	\begin{align*}
	\max_{e\notin \dom(\Psi_i^{\pi})} \E_{\psi'_{t'} \sub \Psi_i^{\pi} }\big[ 
	\Delta(e|\Psi_i^{\pi})\big]
	&\leq (1-\eps) \E_{\psi'_{t'} \sub \Psi_i^{\pi} }\big[ \max_{e\notin 
	\dom(\Psi_i^{\pi})} \Delta(e|\Psi_i^{\pi})\big] & \text{information gap}\\
	& \leq (1-\eps) \Delta'_t. &\text{by adaptive monotonicity}
	\end{align*}
	This implies that for all $e\notin \dom(\Psi_i^{\pi})$, with probability at least $\frac{\eps}{2}$, we have $\Delta(e|\Psi_i^{\pi}) \leq (1-\frac{\eps}{2}) \Delta'_t$. Therefore, for any element $e\in S_{t'}$, we have $e \notin S_{t'+1}$ with probability at least $\frac{\eps}{2}$, as promised.
\end{proof}

Now, we are ready to prove the main theorem of this section. In the following 
theorem, $\pi$ is the semi-adaptive greedy policy, $\pi_{[\ell]}$ is a policy that 
runs $\pi$ and stops if it selects $\ell$ items and $\pi_{[\ell]}^T$ is a policy that 
runs $\pi_{[\ell]}$ and stops if it makes $T$ batch queries.

\begin{theorem}\label{thm:submo:round}
	Let $\pi$ be the semi-adaptive greedy policy and let $\pi_{[\ell]}^T$ be a policy that runs $\pi_{[\ell]}$ and stops if it makes $T$ batch queries.
	For any policy $\pi^*$ (including the optimum policy) and any positive integer $\ell$ we have 
	\begin{align*}
	f_{avg}(\pi_{[\ell]}^T) > \big(1-e^{-\frac{\ell}{k}}-3\eps\big) f_{avg}(\pi^*),
	\end{align*}
	for some $T\in O_{\eps}(\log n \log {\ell})$.
\end{theorem}

\begin{proof}
	Let us set $\delta=\frac {\eps}{\log_{\frac 1 {1 -\eps/2}} (\frac {\ell} {\eps})}$ 
	and let $\Delta_1(\pi)$ be the expected marginal benefit of the first selected item. By 
	applying Lemma \ref{lm:superRound} iteratively $\log_{\frac 1 {1 -\eps/2}} 
	(\frac {\ell} {\eps})$ times we have 
	$$\max_{e\notin \dom(\Psi'_{T})} 
	\Delta(e|\Psi'_{T}) \leq \big(1-\frac {\eps} 2\big)^{\log_{\frac 1 {1 -\eps/2}} 
	(\frac {\ell} {\eps})} \Delta_1(\pi) = \frac {\eps}{{\ell}}\Delta_1(\pi),$$
with probability $1-\delta \times \log_{\frac 1 {1 -\eps/2}} (\frac {\ell} {\eps}) = 
1-\eps$. This means that with probability $(1-\eps)$ the total expected marginal benefit 
of 
the elements added after the $T$-th batch query is at most $${\ell} \times \frac 
{\eps}{{\ell}}\Delta_1(\pi) =  \eps \Delta_1(\pi) \leq \eps f_{avg}(\pi_{[\ell]}),$$ where $T=\log_{\frac 1 {1-\eps/ 2}} \big(\frac{n}{\delta}\big) \times \log_{\frac 1 
{1 -\eps/2}} (\frac {\ell} {\eps})\in O_{\eps}(\log n \log {\ell})$. This, together with 
lemma \ref{lm:par:1-1/e^l}, implies that if we stop policy $\pi$ after $T\in 
O_{\eps}(\log n \log {\ell})$ batch queries for any policy $\pi^*$ we have 
	\begin{align*}
	f_{avg}(\pi_{[\ell]}) > \big(1-e^{-\frac{\ell}{\ex{k}}}-3 \eps\big) f_{avg}(\pi^*),
	\end{align*}
	as desired.
	%XXX Set $\eps=\epsilon/3$ and use $\epsilon$ in the theorem.
\end{proof}

\section{Semi Adaptive Stochastic Minimum Cost Coverage}\label{sec:batch:setcover}
In this section, we bound the efficiency and round complexity of the 
semi-adaptive greedy policy, outlined in Algorithm~\ref{alg:semicover}. %In this section we use $\pi$ to refer to this policy.
To simplify the proofs, we use a more restricted notion of information gap. It is easy to observe that the same proofs in the previous section hold using this version of information gap as well.\footnote{We use this notion in this section for simplicity. However, since the previous notion of information gap is more intuitive, we keep the previous notion as well.}

\begin{algorithm}[H]
	\centering
	\caption{Semi Adaptive Greedy Policy for Minimum Cost Coverage\label{alg:semicover}} 
	\begin{algorithmic}[1]
		\Require Ground set $E$, quota $Q$,  distribution $p(\phi)$, function $f(\cdot)$, value $\epsilon>0$
		% \Procedure{\bmm}{}
		\State {\bf initialize}  $A\leftarrow \emptyset$, $\psi'\leftarrow \emptyset$, $i\leftarrow 1$
		\While{$f(A,\psi')<Q$}
		\While{$\rig(i,\psi')\geq (1-\epsilon)$} 
		\State $e^* = \arg\max_{e\in E\setminus A} \sav(e,i, \psi')$
		\State $A\leftarrow A\cup\{e^*\}$
		\EndWhile
		\State $\psi''\leftarrow$ query all elements in $A\setminus\dom(\psi')$
		\State $\psi'\leftarrow \psi' \cup \psi''$
		\EndWhile
		\State \Return A
		
		% \EndProcedure
	\end{algorithmic}    
\end{algorithm}

\begin{definition}[Restricted Information Gap]\label{def:RestictedInfoGap} At any step $i$ of the policy $\pi$, and given the partial realization $\psi'\sub \Psi_i^{\pi}$, the restricted information gap is defined as follows.
	\begin{align*}
	\rig(i,\psi')=\frac
	{\max_{e\notin \dom(\psi')} \Delta(e|\psi')}
	{\E_{\psi' \sub \Psi_i^{\pi} }\big[ \max_{e\notin \dom(\Psi_i^{\pi})} 
	\Delta(e|\Psi_i^{\pi})\big]}.
	\end{align*}
\end{definition}

\subsection{Performance}

Next theorem bounds the performance of our policy. 
\begin{theorem}\label{thm:perf:setcover}
	Assume that there is a value $\eta \in (0,Q]$ such that $f(\psi)>Q-\eta$ implies $f(\psi)=Q$ for all $\psi$. Let $\pi^*$ be an arbitrary policy that covers everything, i.e., $f(\pi^*)=Q$ for all $\phi$. Let $\pi$ be the semi-adaptive greedy policy, outlined in Algorithm~\ref{alg:semicover}. We have
	\begin{align*}
		{c_{avg}(\pi)} \leq \Big(\frac{c_{avg} (\pi^*)+1}{1-\eps}\Big)\log\Big(\frac{nQ}{\eta}\Big)+1.%\footnote{It is easy to observe that the base of logarithm tends to $e$ as $\epsilon$ tends to zero. However, we use base $2$ for simplicity of presentation.}
	\end{align*}
	%where $n=|E|$.% is the size of the ground set.
\end{theorem}
The proof of this theorem is a combination of the ideas in 
Lemma~\ref{lm:1-1/e^l}, Lemma~\ref{lm:par:1-1/e^l} and 
Theorem~\ref{thm:setcovermain} and is presented in Appendix~\ref{sec:omitted}.

\subsection{Query Complexity}
Next, we bound the number of batch queries of the semi-adaptive greedy policy.

\begin{theorem}\label{thm:round:setcover}
	Assume that there is a value $\eta \in (0,Q]$ such that $f(\psi)>Q-\eta$ implies $f(\psi)=Q$ for all $\psi$. Let $\pi^*$ be any policy that covers everything, i.e., $f(\pi^*)=Q$ for all $\phi$. 
	Let $\pi$ be the semi-adaptive greedy policy (outlined in Algorithm~\ref{alg:semicover}) and let $\pi^T$ be a policy that runs $\pi$ and stops if it makes $T$ batch queries.
	For some $T\in O\big(\log n \log (Qn/\eta)\big)$ we have $f(\pi^T) =Q$ with probability $1-1/n$.
\end{theorem}
We use Lemma \ref{lm:superRound} presented in the previous section together 
with Theorem \ref{thm:perf:setcover} to prove the above theorem. The proof of 
this theorem follows the proof of Theorem \ref{thm:submo:round} and is 
presented in Appendix~\ref{sec:omitted}.

\section{Hardness}
\begin{theorem}\label{thm:hard}
Any policy with a constant approximation guarantee for adaptive stochastic submodular maximization requires $\Omega(\log n)$ batch queries.
\end{theorem}
\begin{proof}
Consider the following example. We have $n=2^{k-1}-1$ elements, and we want to select $k$ elements. The elements are decomposed into $k$ bags of sizes $1,2,\dots,2^k$, where the decomposition is chosen uniformly at random. The objective function for a set $S$ is the number of distinct bags that elements in $S$ belong to. Whenever we select an element $e$ we see all of the elements that are in the same bag as $e$.  It is easy to see that this function is adaptive monotone and adaptive submodular.

Note that one can iteratively select $k$ elements each with a marginal benefit of $1$ 
and hence the value of the optimum solution of this instance is $k$. Next we 
upper-bound the value of the solution of a policy with $t\in o(\log n)$ 
batch-queries.

Let $B_i$ be the $i$-th batch query and let $b_i=|B_i|$. Note that  the 
marginal gain of each element is either $0$ or $1$. Moreover, all of the 
elements with the marginal gain of $1$ are symmetric. Hence, without loss of 
generality, we assume that $B_i$ is random subset of elements with the marginal 
gain of $1$. Hence, with probability at least $(1-\frac 1 {b_i})$ all of the 
elements 
in $B_i$ belong to the $\log^2 b_i$ largest bags with the marginal gain of $1$. 
Hence, 
the 
expected marginal benefit of batch $B_i$ is at most $(1-\frac 1 {b_i})\log^2 b_i 
+\frac 1 {b_i} b_i \leq \log^2 b_i +1$. Hence the expected value of the solution of 
this policy is at most 
$$\sum_{i=1}^t (\log^2 b_i +1) \leq \sum_{i=1}^t (\log^2 \frac k t +1) = t\log^2 \frac k t + t \in o(k),$$
where the last inequality is due to $t\in o(\log n) = o(k)$.
\end{proof}

Notice that in the hard example provided in the above theorem, we upper bound 
the marginal gain of each batch of size $r$ by $\log^2 r + 1$. Hence if we 
force each batch to query exactly $r$ elements, the expected value of the final 
solution is at most $\frac k r \big(\log^2 r + 1\big) \in O(\frac {k \log^2 r} r)$. 
%This implies the following corollary.

\begin{corollary}
Let $\pi$ be a policy for adaptive stochastic submodular maximization that queries batches of size $r$. The approximation factor of $\pi$ is upper bounded by $O(\frac {\log^2 r} r)$.
\end{corollary}

\section{Conclusion}
In this paper, we re-examined  the required rounds of adaptive observations in 
order to maximize an adaptive submodular function. We proposed an efficient 
batch policy that with $O(\log n \times\log k)$ adaptive rounds of observations 
can achieve a  $(1-1/e-\eps)$ approximation guarantee with respect to an 
optimal policy that carries out $k$ actions, from a set of $n$ actions, in a fully 
sequential setting. We also extended  our result to the case of adaptive stochastic 
minimum cost coverage and proposed a batch policy that provides the same 
guarantee in polylogarithmic adaptive rounds through a similar 
information-parallelism scheme. In the mean time, we also proved  the conjecture 
by \cite{golovin11} that the greedy policy achieves the asymptotically tight 
logarithmic approximation guarantee for adaptive stochastic minimum 
cost coverage. One interesting future direction is to develop a semi adaptive policy for maximizing the value of information \cite{chen2015sequential}.

\newpage
\bibliographystyle{abbrv}
\bibliography{references}

\newpage
\appendix
\section{Omitted proofs}\label{sec:omitted}
\subsection{Proof of Corollary~\ref{cr:setcovermain2}}
\begin{proofof}{Corollary \ref{cr:setcovermain2}}%{Theorem~\ref{thm:setcovermain}}
	Let $K$ be a random variable that indicates the number of items picked by 
	$\pi^*$, i.e., $K=S(\pi^*, \Phi)$. Set $\ell=(\ex{K}+1)  \log (\frac{Q}{ \delta\eta})$.
	Note that by definition of $\pi^*$ we have $f(\pi^*)=Q$ for all $\phi$, hence we have $f_{avg}(\pi^*) =Q$.
	By Lemma \ref{lm:1-1/e^l} we have
	\begin{align*}
	f_{avg}(\pi^{\tau_{\ell}}) &> \Big(1-e^{-\frac{\ell}{\ex{K}+1}}\Big) f_{avg}(\pi^*) &\text{By Lemma \ref{lm:1-1/e^l}} \\
	&= \Big(1-e^{-\log (\frac{Q}{ \delta\eta})}\Big) f_{avg}(\pi^*) &\text{Since 
	}\ell=\big(\ex{K}+1\big)  \log (\frac{Q}{ \delta\eta}) \\
	&= \left(1-\frac {\delta \eta} { Q}\right)f_{avg}(\pi^*)\\
	&= Q - \delta \eta. &\text{Since }f_{avg}(\pi^*) =Q
	\end{align*} 
	Recall that by definition we have 
	$\ex{f(\pi^{\tau_{\ell}})}=f_{avg}(\pi^{\tau_{\ell}})=Q - \delta \eta$. 
	Moreover, by adaptive monotonicity we have $  f(\pi^{\tau_{\ell}})\leq  f(\phi) = 
	Q$. Hence by Markov inequality with probability more than $1-\delta$ we have
	$f(\pi^{\tau_{\ell}})>Q-\eta$.
	%\footnote{Note that if we let $\delta=\min_{\phi} 
	%p (\phi)$, for a deterministic policy having %$f(\pi^{\tau_{\ell}}) > Q-\eta$ with 
	%probability greater than $1-\delta$ implies that we always have 
	%$f(\pi^{\tau_{\ell}}) > Q-\eta$.  This 
	%replaces the factor $n$ in the approximation factor with $\frac{1}{\delta}$. 
	%However, since $\delta$ might be exponentially smaller than $\frac 1 n$, we 
	%find the current statement of the theorem more interesting.}
	By definition of $\eta$ this implies that with probability more than $1-\delta$  we 
	have $f(\pi^{\tau_{\ell}})=Q$. Equivalently, probability of $f(\pi^{\tau_{\ell}})\neq Q$ is less than $\delta$. 
	
	Note that the only source of randomness in $\pi^{\tau_{\ell}}$ is from the randomness of the input. Hence, for any fixed $\phi$ we either have $f(\pi^{\tau_{\ell}}) = Q$ or $f(\pi^{\tau_{\ell}})\neq Q$, deterministically. On the other hand, by definition $\delta=\min_{\phi} p (\phi)$. Hence, since the probability of $f(\pi^{\tau_{\ell}})\neq Q$ is less than $\delta$, the probability of $f(\pi^{\tau_{\ell}})\neq Q$ must be $0$.
	Therefore, the policy $\pi^{\tau_{\ell}}$ reaches 
	the utility $Q$, certainly, after selecting  $\ell =\big(\ex{K}+1\big)\log (nQ/\eta)$ items 
	in 
	expectation.
\end{proofof}

%%%%%%%%%%%%%%%%%%%%%%%%%%%%%%%%%%%%%%%%%%%%%%%%%%%%
%%%%%%%%%%%%%%%%%%%%%%%%%%%%%%%%%%%%%%%%%%%%%%%%%%%%

%\begin{comment}

\subsection{Proof of Lemma \ref{lm:par:1-1/e^l}}

\begin{proofof}{Lemma~\ref{lm:par:1-1/e^l}}
	First we provide an upper bound on $f_{avg}(\pi^*)$. Pick an arbitrary $i\in \{0,\dots,\ell\}$. Note that by adaptive monotonicity, we have $f_{avg}(\pi^*) \leq f_{avg}(\pi_{[i]}@\pi^*)$. Next we show that $f_{avg}(\pi_{[i]}@\pi^*) \leq f_{avg}(\pi_{[i]})+\frac k {1-\eps}\Delta_i(\pi)$, where $\Delta_i(\pi)=f_{avg}(\pi_{[i+1]})-f_{avg}(\pi_{[i]})$. This implies that 
	\begin{align}\label{eq:main2}
	f_{avg}(\pi^*) \leq f_{avg}(\pi_{[i]}@\pi^*) \leq f_{avg}(\pi_{[i]})+\frac k {1-\eps}\Delta_i(\pi).
	\end{align}

	Let $\Psi_i$ be a random variable that indicates the partial realization of the elements selected by $\pi_{[i]}$. We use $\psi_i$ to indicate a realization of $\Psi_i$. Let $\bar \Psi$ be a random variable that indicates the last partial realization that is queried by $\pi_{[i]}$ (ignoring the selected elements in the last batch that is not queried yet).
%	For $j\geq i$ let $\Psi_j$ be a random variable that indicates the partial realization of $\pi_{[i]}@\pi^*_{[j-i]}$. We use $\psi_j$ to indicate a realization of $\Psi_j$.
	%
%	We define a random variable $\Gamma(e|\psi_j) = f(\psi_j\cup \Phi(e))-f(\psi_j)$. Note that by definition $\Delta(e|\psi_j)=\ex{\Gamma(e|\psi_j)}$. Let $X_{e,\psi_j}$ be a binary random variable that is $1$ if and only if policy $\pi_{[i]}@\pi^*$ observes subrealization $\psi_j$ and picks item $e$ right after. Note that when the policy is deciding whether to pick $e$ or not, it is not aware of the actual outcome of $\Gamma(e|\psi_j)$. Hence $X_{e,\psi_j}$ and $\Gamma(e|\psi_j)$ are independent.
%	We define $p(\psi_j)$ to be the probability that policy $\pi_{[i]}@\pi^*$ observes subrealization $\psi_j$. 

	%Next we use the notations $\1$ and $\theta$ as defined in Section~\ref{sec:setcover}. For the simplicity of notation we drop the notation of the policy from the subscripts and define $\theta:=\theta_{\pi^{\tau_i}@\pi^*}$ and $\1_{\psi_i,\psi,e,\theta}:=\1_{\psi_i,\psi,e,\pi^{\tau_i}@\pi^*,\theta_{\pi^{\tau_i}@\pi^*}}$.
	We have

	\begin{align*}
	&f_{avg}(\pi_{[i]}@\pi^*) - f_{avg}(\pi_{[i]}) \\
		&= \sum_{\psi_i\in \Psi_i}\sum_{\psi_i \sub \psi'}\sum_{e\notin \dom(\psi')}\sum_{\theta\in \Theta_{\pi^*}} p(\psi')p(\theta)\1_{\psi_i,\psi',e,\theta}\ex{f(\psi'\cup \Phi(e))-f(\psi')|\psi'\sub \Phi \text{ and } \Theta=\theta} &\text{By Definition}\\ 
	&= \sum_{\psi_i\in \Psi_i}\sum_{\psi_i \sub \psi'}\sum_{e\notin \dom(\psi')}\sum_{\theta\in \Theta_{\pi^*}} p(\psi')p(\theta)\1_{\psi_i,\psi',e,\theta}\ex{f(\psi'\cup \Phi(e))-f(\psi')|\psi' \sub \Phi} &\text{Independency}\\ 
	&=  \sum_{\psi_i\in \Psi_i}\sum_{\psi_i \sub \psi'}\sum_{e\notin \dom(\psi')} \sum_{\theta\in \Theta_{\pi^*}} p(\psi')p(\theta)\1_{\psi_i,\psi',e,\theta}\Delta(e|\psi') &\text{Def. of $\Delta(e|\psi')$}\\
	&\leq \sum_{\psi_i\in \Psi_i}\sum_{\psi_i \sub \psi'}\sum_{e\notin\dom(\psi')} \sum_{\theta\in \Theta_{\pi^*}} p(\psi')p(\theta)\1_{\psi_i,\psi',e,\theta}\Delta(e|\psi_i) &\text{Adaptive Submo.}\\
	&\leq \sum_{\psi_i\in \Psi_i}\sum_{\psi_i \sub \psi'}\sum_{e\notin\dom(\psi')} \sum_{\theta\in \Theta_{\pi^*}} p(\psi')p(\theta)\1_{\psi_i,\psi',e,\theta}\max_{e'\notin \dom(\psi_i)}\Delta(e'|\psi_i)\\
	&= \sum_{\psi_i\in \Psi_i}\Big(\sum_{\psi_i \sub \psi'}\sum_{e\notin\dom(\psi')} \sum_{\theta\in \Theta_{\pi^*}} p(\psi')p(\theta)\1_{\psi_i,\psi',e,\theta}\Big)\max_{e'\notin \dom(\psi_i)}\Delta(e'|\psi_i)\\
	&\leq \sum_{\psi_i\in \Psi_i} p(\psi_i) k \max_{e'\notin \dom(\psi_i)}\Delta(e'|\psi_i) \\
	&= k \sum_{\psi_i\in \Psi_i} p(\psi_i) \max_{e'\notin \dom(\psi_i)}\Delta(e'|\psi_i)\\
	&= k \E_{\Psi_i}\Big[\max_{e'\notin \dom(\Psi_i)}\Delta(e'|\Psi_i) \Big]\\
	&= k \E_{\bar \Psi}\Big[\E_{ \bar \Psi \sub \Psi_i}\Big[\max_{e'\notin \dom(\Psi_i)}\Delta(e'|\Psi_i) \Big]\Big] \\
	&\leq k \E_{\bar \Psi}\Big[\frac 1 {1-\eps}\max_{e'\notin \dom(\Psi_i^{\pi})}\E_{ \bar \Psi \sub \Psi_i^{\pi}}\big[\Delta(e'|\Psi_i^{\pi}) \big]\Big] & \text{Information Gap} \\
	& = \frac k {1-\eps} \E_{\bar \Psi}\Big[\max_{e'\notin \dom(\Psi_i^{\pi})}\E_{\bar \Psi \sub \Psi_i^{\pi} }\big[\Delta(e'|\Psi_i^{\pi}) \big]\Big] \\
	& = \frac k {1-\eps} \Delta_i(\pi) \\
	\end{align*}
	This proves Inequality \ref{eq:main2} as promised.	
	Let us define $$\Delta_i^*=f_{avg}(\pi^*)-f_{avg}(\pi_{[i]}).$$ Inequality \ref{eq:main2} implies that $$\Delta_{i}^*\leq \frac k {1-\eps} (\Delta_{i}^*-\Delta_{i+1}^*).$$ By a simple rearrangement we have $$\Delta_{i+1}^*\leq\big(1-\frac {1-\eps} {k}\big)\Delta_{i}^*.$$ By iteratively applying this inequality we have $$\Delta_{\ell}^*\leq\big(1-\frac {1-\eps} {k}\big)^{\ell}\Delta_{0}^*\leq e^{-\frac {(1-\eps)\ell} {k}} \Delta_{0}^*.$$ By applying the definition of $\Delta_{i}^*$ and some rearrangements we have $$f_{avg}(\pi_{[\ell]}) > \big(1-e^{-\frac{(1-\eps)\ell}{k}}\big) f_{avg}(\pi^*)\geq \big(1-e^{-\frac{\ell}{k}}-\eps\big) f_{avg}(\pi^*)$$ as desired.
	
\end{proofof}

%\end{comment}

%%%%%%%%%%%%%%%%%%%%%%%%%%%%%%%%%%%%%%%%%%%%%%%%%%%%
%%%%%%%%%%%%%%%%%%%%%%%%%%%%%%%%%%%%%%%%%%%%%%%%%%%%

\subsection{Lemma~\ref{lm:batch:1-1/e^l}}

Let us start with some definitions. Let $\pi$ be the semi-adaptive greedy policy as defined in section~\ref{sec:batch:kcover}, using restricted information gap (Definition \ref{def:RestictedInfoGap}). For an arbitrary number $\tau$ let $\pi^{\tau}$ be a policy that selects elements according to $\pi$ and stops when the semi-adaptive value of all of the remaining elements is less than or equal to $\tau$. We define $\tau_{i}$ to be a number such that the expected number of elements selected by $\pi^{\tau_i}$ is $i$. 

%In the following lemma we use the notations $\1$ and $\theta$ as defined in Section~\ref{sec:setcover}. For the simplicity of notation we drop the notation of the policy from the subscripts and define $\theta:=\theta_{\pi^{\tau_i}@\pi^*}$ and $\1_{\psi_i,\psi,e,\theta}:=\1_{\psi_i,\psi,e,\pi^{\tau_i}@\pi^*,\theta_{\pi^{\tau_i}@\pi^*}}$.

Now we are ready to prove Lemma~\ref{lm:batch:1-1/e^l}.

\begin{lemma}\label{lm:batch:1-1/e^l}
	For any policy $\pi^*$ and any positive integer $\ell$ we have 
	\begin{align*}
	f_{avg}(\pi^{\tau_{\ell}}) > \big(1-e^{-\frac{(1-\eps)\ell}{\ex{K}+1}}\big) f_{avg}(\pi^*),
	\end{align*}
	where $K$ is a random variable that indicates the number of items picked by $\pi^*$, i.e. $K=|S(\pi^*, \Phi)|$.
\end{lemma}
\begin{proof}
    Let us define $\Delta^{\tau_i}(\pi)=f_{avg}(\pi^{\tau_{i}})-f_{avg}(\pi^{\tau_{i-1}})$. 
    Recall that in expectation $\pi^{\tau_{i}}$ picks one item more than $\pi^{\tau_{i-1}}$. Moreover note that by definition of $\pi^{\tau_{i}}$ the semi-adaptive value of all of the items selected by $\pi^{\tau_{i}}$ is at least $\tau_i$. Hence 
\begin{align}\label{eq:tauVSdelta}
    f_{avg}(\pi^{\tau_{i}})-f_{avg}(\pi^{\tau_{i-1}}) = \Delta^{\tau_i}(\pi) \geq \tau_i.
\end{align}

	%Note that for any arbitrary $i\in \{0,\dots,\ell\}$ by adaptive monotonicity, we have $f_{avg}(\pi^*) \leq f_{avg}(\pi^{\tau_i}@\pi^*)$. 
	Let $\Psi_i$ be a random variable that indicates the partial realization of the elements selected by $\pi^{\tau_i}$. We use $\psi_i$ to indicate a realization of $\Psi_i$. %Let $\bar \Psi$ be a random variable that indicates the last partial realization that is queried in $\pi^{\tau_{i}}$.
%
	%We define a random variable $\Gamma(e|\psi') = f(\psi'\cup \Phi(e))-f(\psi')$. Note that by definition $\Delta(e|\psi')=\ex{\Gamma(e|\psi')}$. Now consider the policy $\pi^{\tau_i}@\pi^*$. Let $X_{e,\psi'}$ be a binary random variable that is $1$ if and only if policy $\pi^{\tau_i}@\pi^*$ observes subrealization $\psi'$ and picks item $e$ right after.
%	
	Next we show that for any consistent partial realization $ \psi_i \sub \psi'$ and any element $e$ we have
	\begin{align}\label{eq:Delta<Tau}
	\Delta(e|\psi') \leq \frac{\tau_{i}}{1-\eps}.
	\end{align}
	We have two cases based on the time that the policy $\pi^{\tau_i}$ stops.
	\begin{itemize}
	\item The policy $\pi^{\tau_i}$ queries a batch and then observe that the semi-adaptive value of all items drop below $\tau_i$ and then $\pi^{\tau_i}$ stops. 
	\item While adding items to a batch (and before performing the query), the semi-adaptive value of all items drop below $\tau_i$ and then $\pi^{\tau_i}$ stops. 
	\end{itemize}
	Note that in the first case the semi-adaptive values of all of the items are equal to their actual expected marginal benefit (i.e., $\Delta(e|\psi')$). Hence, we have $\Delta(e|\psi') \leq \Delta(e|\psi_i) \leq \tau_{i}$. In the second case, by the definition of the algorithm, the restricted information gap is at least $1-\epsilon$. This together with the fact that the semi-adaptive values of all items are below $\tau_i$ implies that the expected marginal benefit of the first item that was added to the last batch is at most $\frac{\tau_i}{1-\epsilon}$. This together with adaptive monotonicity implies $\Delta(e|\psi') \leq \frac{\tau_{i}}{1-\eps}$ as desired.

	We have
		\begin{align*}
	&f_{avg}(\pi^{\tau_i}@\pi^*) - f_{avg}(\pi^{\tau_i}) \\
	&= \sum_{\psi_i\in \Psi_i}\sum_{\psi_i \sub \psi'}\sum_{e\notin \dom(\psi')}\sum_{\theta\in \Theta_{\pi^*}} p(\psi')p(\theta)\1_{\psi_i,\psi',e,\theta}\ex{f(\psi'\cup \Phi(e))-f(\psi')|\psi' \sub \Phi  \text{ and } \Theta=\theta} &\text{By Definition}\\ 
	&= \sum_{\psi_i\in \Psi_i}\sum_{\psi_i \sub \psi'}\sum_{e\notin \dom(\psi')}\sum_{\theta\in \Theta_{\pi^*}} p(\psi')p(\theta)\1_{\psi_i,\psi',e,\theta}\ex{f(\psi'\cup \Phi(e))-f(\psi')|\psi' \sub \Phi} &\text{Independency}\\ 
	&=  \sum_{\psi_i\in \Psi_i}\sum_{\psi_i \sub \psi'}\sum_{e\notin \dom(\psi')} \sum_{\theta\in \Theta_{\pi^*}} p(\psi')p(\theta)\1_{\psi_i,\psi',e,\theta}\Delta(e|\psi') &\text{Definition of $\Delta(e|\psi')$}\\
	&\leq \sum_{\psi_i\in \Psi_i}\sum_{\psi_i \sub \psi'}\sum_{e\notin \dom(\psi')} \sum_{\theta\in \Theta_{\pi^*}} p(\psi')p(\theta)\1_{\psi_i,\psi',e,\theta}\frac{\tau_i}{1-\eps} &\text{Inequality \ref{eq:Delta<Tau}}\\
	&\leq \sum_{\psi_i\in \Psi_i}\sum_{\psi_i \sub \psi'}\sum_{e\notin \dom(\psi')} \sum_{\theta\in \Theta_{\pi^*}} p(\psi')p(\theta)\1_{\psi_i,\psi',e,\theta}\frac{\Delta^{\tau_i}(\pi)}{1-\eps} &\text{Inequality \ref{eq:tauVSdelta}} \\
	&= \Big(\sum_{\psi_i\in \Psi_i}\sum_{\psi_i \sub \psi'}\sum_{e\notin \dom(\psi')} \sum_{\theta\in \Theta_{\pi^*}} p(\psi')p(\theta)\1_{\psi_i,\psi',e,\theta}\Big)\frac{\Delta^{\tau_i}(\pi)}{1-\eps} \\
	& = \frac{\ex{K}}{1-\eps}\Delta^{\tau_i}(\pi).
	\end{align*}
	Now define $\Delta_i^*=f_{avg}(\pi^*)-f_{avg}(\pi^{\tau_i})$. The above inequality implies that $$\Delta_{i}^*\leq \frac{\ex{K}}{1-\eps}(\Delta_{i-1}^*-\Delta_{i}^*).$$ By a simple rearrangement we have $$\Delta_{i}^*\leq\big(1-\frac 1 {\frac{\ex{K}}{1-\eps}+1}\big)\Delta_{i-1}^*\leq \big(1-\frac {1-\eps} {{\ex{K}}+1}\big)\Delta_{i-1}^*.$$ By iteratively applying this inequality we have $$\Delta_{\ell}^*\leq\big(1-\frac {1-\eps} {\ex{K}+1}\big)^{\ell}\Delta_{0}^*\leq e^{-\frac{(1-\eps)\ell}{\ex{K}+1}} \Delta_{0}^*.$$ By applying the definition of $\Delta_{i}^*$ and some rearrangement we have $$f_{avg}(\pi^{\tau_{\ell}}) > \big(1-e^{-\frac{(1-\eps)\ell}{\ex{K}+1}}\big) f_{avg}(\pi^*)$$ as desired.
\end{proof}

%%%%%%%%%%%%%%%%%%%%%%%%%%%%%%%%%%%%%%%%%%%%%%%%%%%%
%%%%%%%%%%%%%%%%%%%%%%%%%%%%%%%%%%%%%%%%%%%%%%%%%%%%

\subsection{Proof of Theorem \ref{thm:perf:setcover}}

\begin{proofof}{Theorem~\ref{thm:perf:setcover}}
	Let $K$ be a random variable that indicates the number of items picked by $\pi^*$. Set $$\ell=\frac{\ex{K}+1}{1-\eps}  \log (nQ/\eta).$$
	Note that by definition of $\pi^*$ we have $f(\pi^*)=Q$ for all $\phi$, hence we have $f_{avg}(\pi^*) =Q$.
	By Lemma \ref{lm:batch:1-1/e^l} we have
	\begin{align*}
	f_{avg}(\pi^{\tau_{\ell}}) &> \Big(1-e^{-\frac{(1-\eps)\ell}{\ex{K}+1}}\Big) f_{avg}(\pi^*) &\text{By Lemma \ref{lm:batch:1-1/e^l}} \\
	&= \Big(1-e^{-log (nQ/\eta)}\Big) f_{avg}(\pi^*) &\text{Since }\frac{\ex{K}+1}{1-\eps}  \log (nQ/\eta) \\
	&= (1-\frac \eta {n Q})f_{avg}(\pi^*)\\
	&= Q - \frac \eta n. &\text{Since }f_{avg}(\pi^*) =Q
	\end{align*} 
	Recall that, by definition $f_{avg}(\pi^{\tau_{\ell}}) = \ex{f(\pi^{\tau_{\ell}})}$. Moreover, note that by adaptive monotonicity we have $  f(\pi^{\tau_{\ell}})\leq  f(\phi) = Q$. Hence by Markov inequality with probability $1-1/n$ we have
	$f(\pi^{\tau_{\ell}})>Q-\eta$.
	By definition of $\eta$ this implies that with probability $1-\frac 1 n$  we have $f(\pi^{\tau_{\ell}})=Q$.
	Therefore, with probability $1-1/n$, $\pi^{\tau_{\ell}}$ reaches $Q$ after selecting  $\ell =\big(\ex{K}+1\big)\log (nQ/\eta)$ items in expectations. Otherwise, $\pi$ picks at most all $n$ items. Hence the expected number of items that $\pi$ picks is upper bounded by
	\begin{align*}
	(1-\frac 1 n) \times \frac{\ex{K}+1}{1-\eps}  \log (\frac{nQ}{\eta}) +  \frac 1 n \times n  \leq  \big(\frac{c_{avg} (\pi^*)+1}{1-\eps}\big)  \log\Big(\frac{nQ}{\eta}\Big)+1.
	\end{align*}
	
\end{proofof}

%%%%%%%%%%%%%%%%%%%%%%%%%%%%%%%%%%%%%%%%%%%%%%%%%%%%
%%%%%%%%%%%%%%%%%%%%%%%%%%%%%%%%%%%%%%%%%%%%%%%%%%%%

\subsection{Proof of Theorem \ref{thm:round:setcover}}
 First let us start with a couple of definitions. We define random variable $\bar \Psi_t$ to be the partial realization obtained by the $t$-th query. Next we prove Theorem \ref{thm:round:setcover}.

\begin{proofof}{Theorem~\ref{thm:round:setcover}}   
	In order to prove this theorem we show that $f_{avg}(\pi^T) > Q-\frac{\eta}{n}$. Note that $f(\pi^T)\leq Q$. This together with a Markov bound imply $f(\pi^T) > Q-\eta$, and hence $f(\pi^T)= Q$, with probability $1-1/n$. In this theorem we simply set $\eps=0.01$.
	
	Let us set $$\delta=\frac {\frac{\eta}{2Qn}}{\log_{\frac 1 {1 -\eps/2}} \frac {2Qn^2} {\eta}}.$$ By applying Lemma \ref{lm:superRound}, $\log_{\frac 1 {1 -\eps/2}} \frac {2Qn^2} {\eta}$ times iteratively we have $$\max_{e\notin \dom(\bar \Psi_{T})} \Delta(e|\bar \Psi_{T}) \leq \big(1-\frac {\eps} 2\big)^{\log_{\frac 1 {1 -\eps/2}} \frac {2Qn^2} {\eta}} \Delta_1(\pi) = \frac {\eta}{2Qn^2}\Delta_1(\pi),$$
	with probability $1-\delta \times \log_{\frac 1 {1 -\eps/2}} \frac {2Qn^2} {\eta} = 1-\frac{\eta}{2Qn}$. This means that with probability $1-\frac{\eta}{2Qn}$ the total expected marginal gain of the elements added after the $T$-th query is at most $${n} \times \frac {\eta}{2Qn^2}\Delta_1(\pi) =  \frac{\eta}{2Qn} \Delta_1(\pi) \leq \frac{\eta}{2n},$$ where 
	\begin{align*}
	T&=\log_{\frac 1 {1-\eps/ 2}} \big(\frac{n}{\delta}\big) \times \log_{\frac 1 {1 -\eps/2}} \frac {2Qn^2} {\eta} \\&\in O\big((\log n + \log \log (Qn/\eta))\log (Qn/\eta)\big)\\
	&\in O\big(\log n \log (Qn/\eta)\big).\text{\footnotemark}
	\end{align*}
	This implies that $$f_{avg}(\pi^T) \geq (1-\frac{\eta}{2Qn})(Q-\frac{\eta}{2n}) > Q-\frac{\eta}{n}.$$
	\footnotetext{We can assume $\log n > \log \log (Qn/\eta)$, since otherwise $ n \leq \log (Qn/\eta)$ and hence trivially $T\in O(\log (Qn/\eta))$ as desired.}
	%	\begin{align*}
	%	f_{avg}(\pi_{[\ell]}) > \big(1-e^{-\frac{\ell}{\ex{k}}}-3 \eps\big) f_{avg}(\pi^*),
	%	\end{align*}
	This implies that $f(\pi^T) = Q$ with probability at least $1-\frac{1}{n}$, as desired. 
\end{proofof}

%%%%%%%%%%%%%%%%%%%%%%%%%%%%%%%%%%%%%%%%%%%%%%

\section{Truncation}\label{appx:trunc}

Consider the following simple adaptive submodular function $f(\cdot)$. We have three elements $\{x,y,z\}$ each of $x$ and $y$ are associated with an independent uniform random binary variable. 
The value of the empty set is zero. If element $z$ exists in a set, it deterministically adds a value $1$ to the set. If there is only one of $x$ and $y$ in the set, it adds a value $1$ to the set. However, if both $x$ and $y$ are in the set, if their corresponding random variables match, they add a value $2$ to the set, and otherwise add nothing. 

Note that if one of $x$ and $y$ exists in a set, adding the other one does not change the value of the set in expectation. In all other cases the value of adding an element is $1$. This implies that this function is adaptive submodular.
However, $g(\psi)=\min\big(f(\psi),1\big)$ is not adaptive submodular. For example the marginal gain of $z$ on $g\big(\{(x,1)\}\big)$ is $0$ but the marginal gain of $z$ on $g\big(\{(x,1),(y,0)\}\big)$ is $1$.

\end{document}